%% file: sample-sigconf.tex
\documentclass[sigconf]{acmart}
\input{usepackages}

\AtBeginDocument{%
  \providecommand\BibTeX{{%
    \normalfont B\kern-0.5em{\scshape i\kern-0.25em b}\kern-0.8em\TeX}}}

\copyrightyear{2024}
\acmYear{2024}
\setcopyright{acmlicensed}
\acmConference[WWW '24]{Proceedings of the ACM Web Conference 2024}{May 13--17, 2024}{Singapore, Singapore}
\acmBooktitle{Proceedings of the ACM Web Conference 2024 (WWW '24), May 13--17, 2024, Singapore, Singapore}
\acmDOI{10.1145/3589334.3645702}
\acmISBN{979-8-4007-0171-9/24/05}
\begin{document}
\title{Towards Efficient Communication and Secure Federated Recommendation System via Low-rank Training}

\author{Ngoc-Hieu Nguyen}
\email{ngochieutb13@gmail.com}
\orcid{0009-0009-5873-6532}
\affiliation{%
  \institution{College of Engineering \& Computer Science, VinUniversity}
  % \streetaddress{P.O. Box 1212}
  \city{Hanoi}
  \country{Vietnam}
}

\author{Tuan-Anh Nguyen}
\email{21anh.nt@vinuni.edu.vn}
\affiliation{%
  \institution{College of Engineering \& Computer Science, VinUniversity}
  \city{Hanoi}
  \country{Vietnam}
}

\author{Tuan Nguyen}
\email{tuan.nm@vinuni.edu.vn}
\affiliation{%
  \institution{College of Engineering \& Computer Science, VinUniversity}
  \city{Hanoi}
  \country{Vietnam}
}

\author{Vu Tien Hoang}
\email{vu.ht@vinuni.edu.vn}
\affiliation{%
  \institution{College of Engineering \& Computer Science, VinUniversity}
  \city{Hanoi}
  \country{Vietnam}
}

\author{Dung D. Le}
\authornote{Co-last author}
\email{dung.ld@vinuni.edu.vn}
\affiliation{%
  \institution{College of Engineering \& Computer Science, VinUniversity}
  \city{Hanoi}
  \country{Vietnam}
}

\author{Kok-Seng Wong}
\authornotemark[1]
\email{wong.ks@vinuni.edu.vn}
\affiliation{%
  \institution{College of Engineering \& Computer Science, VinUniversity}
  \city{Hanoi}
  \country{Vietnam}
}

\renewcommand{\shortauthors}{Ngoc-Hieu, et al.}

\begin{abstract}
Federated Recommendation (FedRec) systems have emerged as a solution to safeguard users' data in response to growing regulatory concerns.
However, one of the major challenges in these systems lies in the communication costs 
that arise from the need to transmit neural network models between user devices and a central server.
Prior approaches to these challenges often lead to issues such as computational overheads, model specificity constraints, and compatibility issues with secure aggregation protocols. In response, we propose a novel framework, called Correlated Low-rank Structure (\arcronym), which leverages the concept of adjusting lightweight trainable parameters while keeping most parameters frozen. Our approach substantially reduces communication overheads without introducing additional computational burdens. Critically, our framework remains fully compatible with secure aggregation protocols, including the robust use of Homomorphic Encryption. 
The approach resulted in a reduction of up to 93.75\% in payload size, with only an approximate 8\% decrease in recommendation performance across datasets. Code for reproducing our experiments can be found at https://github.com/NNHieu/CoLR-FedRec.
\end{abstract}



% \keywords{Recommendation System; Federated Learning; Communication efficiency}



\maketitle
\input{main_content}

\begin{acks}
This research was funded by Vingroup Innovation Foundation (VINIF) under project code VINIF.2022.DA00087
\end{acks}

\bibliographystyle{ACM-Reference-Format}
\balance
\bibliography{sample-base}

\appendix
\include{appendix}

\end{document}

%% file: usepackages.tex
\usepackage[ruled,vlined]{algorithm2e}
\usepackage{amsmath,bm}

\newcommand{\numClients}{\ensuremath{M}}
\newcommand{\numItems}{\ensuremath{N}}
\newcommand{\embsz}{\ensuremath{d}}
\newcommand{\lorarank}{\ensuremath{r}}

\newcommand{\localStep}{\tau}

\newcommand{\activeClients}{\mathcal{S}}
\newcommand{\sgrad}{\nabla}

\newcommand{\lr}{\eta}
\newcommand{\slr}{\lr_{s}}
\newcommand{\usr}{u}
\newcommand{\itm}{i}
\newcommand{\fullname}{Correlated Low-rank Structure update}
\newcommand{\arcronym}{CoLR}

\newcommand{\serveropt}{\textsc{ServerOpt}\xspace}
\newcommand{\clientopt}{\textsc{ClientOpt}\xspace}

\newcommand{\revise}[1]{#1}

%% file: main_content.tex
\section{Introduction}
In a centralized recommendation system, all user behavior data is collected on a central server for training. However, this method can potentially expose private information that users may be hesitant to share with others. As a result, various regulations such as the General Data Protection Regulation (GDPR)\citep{voigt2017eu} and the California Consumer Privacy Act (CCPA)\citep{pardau2018california} have been implemented to limit the centralized collection of users' personal data. In response to this challenge, and in light of the increasing prevalence of edge devices, federated recommendation (FedRec) systems have gained significant attention for their ability to uphold user privacy \cite{ammad2019federated, fedrec_2020, chai2020secure, lin2020meta, hegedhu2020, Yang2020, privrec_2021, wang2021poirec, wu2022federated, perifanis2022federated, liu2022fairfedrec}. 

The training of FedRec systems is often in a cross-device setting which involves transferring recommendation models between a central server and numerous edge devices, such as mobile phones, laptops, and PCs.
It is increasingly challenging to transfer these models due to the growing model complexity and parameters in modern recommendation systems \citep{Liao2016clustering, naumov2019deep, yi-etal-2021-efficient}.
In addition, clients participating in FedRec systems often exhibit differences in their computational processing speeds and communication bandwidth capabilities, primarily stemming from variations in their hardware and infrastructure \cite{Li2020flchallenges}. 
These discrepancies can give rise to stragglers and decrease the number of participants involved in training, potentially leading to diminished system performance.

Practical FedRec systems require the implementation of mechanisms that reduce the amount of communication costs. Three commonly used approaches to reduce communication costs include (i) reducing the frequency of communication by allowing local updates, (ii) minimizing the size of the message through message compression, and (iii) reducing the server-side communication traffic by restricting the number of participating clients per round \citep{wang2021fieldguide}. Importantly, these three methods are independent and can be combined for enhanced efficiency.

In this study, we address the challenge of communication efficiency in federated recommendations by introducing an alternative to compression methods. Many existing compression methods involve encoding and decoding steps that can introduce significant delays, potentially outweighing the gains achieved in per-bit communication time \cite{vkj2019powersgd}.
Another crucial consideration is the compatibility with aggregation protocols. For example, compression techniques that do not align with all-reduce aggregation may yield reduced communication efficiency in systems employing these aggregation techniques \citep{vkj2019powersgd}.
This is also necessary for many secure aggregation protocols such as Homomorphic Encryption (HE) \citep{bonawitz2017practicalsecure,acar2018surveyhe}.
Moreover, many algorithms assume that clients have the same computational power, but this may induce stragglers due to computational heterogeneity and can increase the runtime of algorithms.

Based on our observation that the update transferred between clients and the central server in FedRec systems has a low-rank bias (Section \ref{sec:motivation}), we propose \fullname~(\arcronym).
\arcronym\;increases communication efficiency by adjusting lightweight trainable parameters while keeping most parameters frozen. Under this training scheme, only a small amount of trainable parameters will be shared between the server and clients.
Compared with other compression techniques, our methods offer the following benefits. (i) \textbf{Reduce both up-link and down-link communication cost:} \arcronym~avoid the need of unrolling the low-rank message in the aggregation step by using a correlated projection, (ii) \textbf{Low computational overheads:} Our method enforces a low-rank structure in the local update during the local optimization stage so eliminates the need to perform a compression step. Moreover, \arcronym~can be integrated into common aggregation methods such as FedAvg and does not require additional computation. (iii) \textbf{Compatible with secure aggregation protocols: } the aggregation step on \arcronym~can be carried by simple additive operations, this simplicity makes it compatible with strong secure aggregation methods such as HE, (iv) \textbf{Bandwidth heterogeneity awareness:} Allowing adaptive rank for clients based on computational/communication budget. Our framework demonstrates a capability to provide a strong foundation for building a secure and practical recommendation system.

Our contributions can be summarized as following:

\begin{itemize}
    \item We propose a novel framework, \arcronym, designed to tackle the communication challenge in training FedRec systems. 
    \item We conducted experiments to showcase the effectiveness of \arcronym. Notably, even with an update size equates to 6.25\% of the baseline model, \arcronym\;demonstrates remarkable efficiency by retaining 93.65\% accuracy (in terms of HR) compared to the much larger baseline.
    \item We show that CoLR is compatible with HE-based FedRec systems and, hence, reinforces the security of the overall recommendation systems.
\end{itemize}

\section{Related Work}
\paragraph{Federated Recommendation (FedRec) Systems. } 
In recent years, FedRec systems have risen to prominence as a key area of research in both machine learning and recommendation systems. FCF~\cite{ammad2019federated} and FedRec~\cite{fedrec_2020} are the pioneering FL-based methods for collaborative filtering based on matrix factorization. The former is designed for implicit feedback, while the latter is for explicit feedback.
To enhance user privacy, FedMF~\cite{chai2020secure} applies distributed matrix factorization within the FL framework and introduces the HE technique for securing gradients before they are transmitted to the server.
MetaMF~\cite{lin2020meta} is a distributed matrix factorization framework using a meta-network to generate rating prediction models and private item embedding. 
~\cite{wu2022federated} presents FedPerGNN, where each user maintains a GNN model to incorporate high-order user-item information.
FedNCF~\cite{perifanis2022federated} adapts Neural Collaborative Filtering (NCF)~\cite{He2017ncf} to the federated setting, incorporating neural networks to learn user-item interaction functions and thus enhancing the model's learning capabilities. 

\paragraph{Communication Efficient Federated Recommendation.}
Communication efficiency is of the utmost importance in FL \cite{jakub2016strategies}. 
JointRec \cite{duan2020joinrec} reduces uplink costs in FedRS by using low-rank matrix factorization and 8-bit probabilistic quantization to compress weight updates.
Some works explore reducing the entire item latent matrix payload by meta-learning techniques \cite{lin2020meta, privrec_2021}. 
LightFR~\cite{zhang2023lightfr} proposes a framework to reduce communication costs by exploiting the learning-to-hash technique under federated settings and enjoys both fast online inference and reduced memory consumption. Another solution is proposed by \citet{Khan2021apayload}, which is a multi-arm bandit algorithm to address item-dependent payloads.

\paragraph{Low-rank Structured Update}
\citet{jakub2016strategies} propose to enforce every update to local model $\Delta_\usr$ to have a low rank structure by express $\Delta_\usr = A_\usr^{(t)} B_\usr^{(t)}$ where $A_\usr^{(t)} \in \mathbb{R}^{d_1 \times k}$ and $B_\usr^{(t)} \in \mathbb{R}^{k \times d_2}$. In subsequent computation, $A_\usr^{(t)}$ is generated independently for each client and frozen during local training procedures.
This approach saves a factor of $d_1 / k$. 
\citet{hyeon-woo2022fedpara} proposes a method that re-parameterizes weight parameters of layers using low-rank weights followed by the Hadamard product. The authors show that FedPara can achieve comparable performance to the original model with 3 to 10 times lower communication costs on various tasks, such as image classification, and natural language processing.

\paragraph{Secure FedRec.}
Sending updates directly to the server without implementing privacy-preserving mechanisms can lead to security vulnerabilities.~\citet{chai2020secure} demonstrated that in the case of the Matrix Factorization (MF) model using the FedAvg learning algorithm, if adversaries gain access to a user's gradients in two consecutive steps, they can deduce the user's rating information. 
One approach involves leveraging HE to encrypt intermediate parameters before transmitting them to the server \cite{chai2020secure, perifanis2023}. This method effectively safeguards user ratings while maintaining recommendation accuracy. However, it introduces significant computational overhead, including encryption and decryption steps on the client side, as well as aggregation on the server side. Approximately 95\% of the time consumed by the system is dedicated to operators carried out on the ciphertext ~\cite{chai2020secure}.
\citet{Liang2021fedrec++} aim to enhance the performance of FedRec systems using denoising clients.
\citet{liu2021fedct} discuss the development of secure recommendation systems in cross-domain settings.
Recent studies~\cite{wu2022fedattack, zhang2022pipattack, yuan2023manipulating} show that FedRecs are susceptible to poisoning attacks of malicious clients. 

\section{Preliminaries}
In this section, we present the preliminaries and the setting that the paper is working with. Also, this part will discuss the challenges in applying compression methods.

\subsection{Federated Learning for Recommendation}
In the typical settings of item-based FedRec systems~\cite{fedrec_2020}, there are $\numClients$ users and $\numItems$ items where
each user $\usr$ has a private interaction set denoted as $O_\usr = \{(\itm, r_{\itm\usr})\} \subset [\numItems] \times \mathbb{R}$. 
These users want to jointly build a recommendation system based on local computations without violating participants' privacy.
This scenario naturally aligns with the horizontal federated setting~\cite{McMahan2016fedavg}, as it allows us to treat each user as an active participant. 
In this work, we also use the terms user and client interchangeably. 
The primary goal of such a system is to generate a ranked list of top-K items that a given user has not interacted with and are relevant to the user's preferences.
Mathematically, we can formalize the problem as finding a global model parameterized by $\theta$ that minimizes the following global loss function $\mathcal{L}(\cdot)$:
\begin{equation}
    \mathcal{L}(\theta) \triangleq 
    % \mathbb{E}_{k \sim \mathcal{P}}[\mathcal{L}_k(\theta)] 
     \sum_{\usr=1}^M w_\usr \mathcal{L}_\usr(\boldsymbol{\theta})
\end{equation}
where $\theta$ is the global parameter, $w_\usr$ is the relative weight of user $\usr$.
And $\mathcal{L}_\usr(\theta) := \sum_{(\itm, r_{\usr\itm}) \in O_{\usr}} \ell_\usr (\theta, (\itm, r_{\usr\itm}))$ is the local loss function at user $\usr$'s device. 
Here $(\itm, r_{\usr\itm})$ represents a data sample from the user's private dataset, and $\ell_\usr$ is the loss function defined by the learning algorithm.
Setting $w_\usr = N_\usr / N$ where $N_\usr = |O_\usr|$ and $N = \sum_{\usr=1}^M N_\usr$ makes the objective function $\mathcal{L}(\theta)$ equivalent to the empirical risk minimization objective function of the union of all the users' dataset. Once the global model is learned, it can be used for user prediction tasks.

In terms of learning algorithms, Federated Averaging (FedAvg) \citep{McMahan2016fedavg} is one of the most popular algorithms in FL. FedAvg divides the training process into rounds. At the beginning of the $t$-th round ($t \ge 0$), the server broadcasts the current global model $\theta^{(t)}$ to a subset of users $\mathcal{S}^{(t)}$ which is often uniformly sampled without replacement in simulation \cite{wang2021fieldguide, fedrec_2020}. Then each sampled client in the round's cohort performs $\tau_\usr$ local SGD updates on its local dataset and sends the local model changes $\Delta_\usr^{(t)} = \theta_\usr^{(t, \tau_\usr)} - \theta^{(t)}$ to the server. Finally, the server performs an aggregation step to update the global model:
\begin{equation}
    \theta^{(t + 1)} = \theta^{(t)} + \frac{\sum_{\usr \in \mathcal{S}^{(t)}} w_\usr \Delta_\usr^{(t)}}{\sum_{\usr \in \mathcal{S}^{(t)}} w_\usr}
\end{equation}
The above procedure will repeat until the algorithm converges.

\subsection{Limitation of current compression methods}
Communication is one of the main bottlenecks in FedRec systems and can be a serious constraint for both servers and clients. 
Although diverse optimization techniques exist to enhance communication efficiency, such methods may not preserve privacy.
Moreover, tackling privacy and communication efficiency as separate concerns can result in suboptimal solutions.

\paragraph{Top-K compression}   
This method is based on sparsification, which represents updates as sparse matrices to reduce the transfer size.
However, the process of allocating memory for copying the gradient (which can grow to a large size, often in the millions) and then sorting this copied data to identify the top-K threshold during each iteration is costly enough that it negates any potential enhancements in overall training time when applied to real-world systems. As a result, employing these gradient compression methods in their simplest form does not yield the expected improvements in training efficiency. As observed in \citet{gupta2021largescale}, employing the Top-K compression for training large-scale recommendation models takes 11\% more time than the baseline with no compression. 

\paragraph{SVD compression}
This method returns a compressed update with a low-rank structure, which is based on singular value decomposition. 
After obtaining factorization results $U_\usr$ and $V_\usr$, the aggregation step requires performing decompression and computing $\sum_{\usr \in \activeClients} \frac{N_u}{N} U_\usr V_\usr$ and this sum is not necessarily low-rank so there is no readily reducing cost in the downlink communication without additional compression-decompression step.
The need to perform matrix multiplication makes this method incompatible with HE. 
Moreover, performing SVD decomposition on an encrypted matrix by known schemes remains an open problem.

\section{Proposed method}
\begin{figure*}
    \includegraphics[width=1.66\columnwidth]{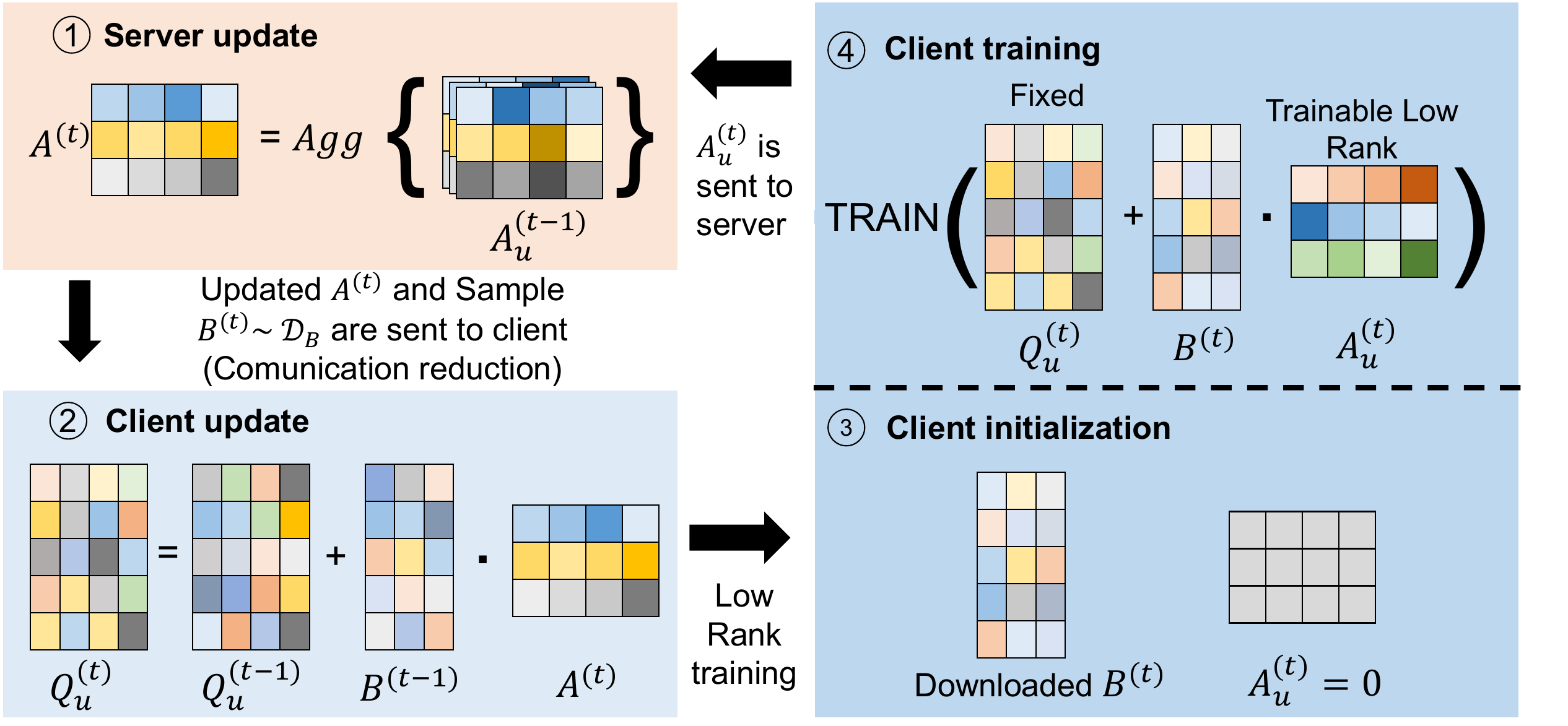}
    \caption{Illustration of CoLR at training round $t$.  At first, the server conducts aggregation over the local model $A_\usr^{(t-1, \tau_\usr)}$ to obtain the global model update $A^{(t)}$. Subsequently, $A^{(t)}$ are transmitted to the clients. The client will update their $Q_\usr^{(t)}$  using this $A^{(t)}$, then initilizes a new matrix  $A_\usr^{(t,0)}$ and download the matrix $B^{(t)}$ which is sampled at the server and shared between clients. Finally, the client carries out local training and then sends the local model update $A_\usr^{(t, \tau_\usr)}$ to the server for the next training round.}
    \label{fig:main-method}
\end{figure*}

\subsection{Motivation}
\label{sec:motivation}

Our method is motivated by analyzing the optimization process at each user's local device.
We consider an effective federated matrix factorization (FedMF) as the backbone model. 
This model represents each item and user by a vector with the size of $\embsz$ denoted $\mathbf{q}_\itm$ and $\mathbf{p}_\usr$ respectively. And the predicted ratings $r_{\usr\itm}$ are given by $\hat{r}_{\usr\itm} = \mathbf{q}_\itm^\top \mathbf{p}_\usr$. 
Then the user-wise local parameter $\theta_\usr$ consists of the user $\usr$'s embedding $\mathbf{p}_\usr$  and the item embedding matrix $Q$, where $\mathbf{q}_\itm$ is the $\itm$th column of $Q$. 
The loss function $\mathcal{L}_\usr$ at user $\usr$'s device is given in the following.
\begin{align*}
    \mathcal{L}_{\usr}(\mathbf{p}_\usr, Q) 
    =& \sum_{(\itm, r_{\usr\itm})\in O_u} \ell \left(r_{\usr\itm}, ({Q^\top\mathbf{p}_u})_{\itm}\right) + \frac{\lambda}{2} \|\mathbf{p}_u\|_2^2 + \frac{\lambda}{2} \|Q\|_2^2  
\end{align*}
Let $\eta$ be the learning rate, the update on the user embedding $\mathbf{p}_u$ at each local optimization step is given by:
\begin{align} 
    \mathbf{p}_u^{(t + 1)} 
    =& \mathbf{p}_u^{(t)}(1 - \eta\lambda) +\eta Q^{(t)\top}(\mathbf{r} - \hat{\mathbf{r}}^{(t)}).
\end{align}
Let $\mathbf{m} \in \mathbb{R}^{\numItems}$ be a binary vector where $\mathbf{m}_\itm = 1$ if $\itm \in O_\usr$, then the item embedding matrix $Q$ is updated as follows:

\begin{equation}
\label{eq:update_Q}
    Q^{(t + 1)} = Q^{(t)} - \eta (\lambda Q^{(t)}  - \left(\mathbf{m} * (\mathbf{r}_u - \hat{\mathbf{r}}_u)\right) \mathbf{p}_u^{(t)\top})
\end{equation}
The update that is sent to the central server has the following formula,
\begin{align}
    \Delta_Q^{(t)} 
    = Q^{(t + 1)} - Q^{(t)} 
    = \eta \left[\left(\mathbf{m} * (\mathbf{r}_u - \hat{\mathbf{r}}_u)\right) \mathbf{p}_u^{(t)\top} - \lambda Q^{(t)} \right]
\end{align}
As we can see from equation \ref{eq:update_Q}, since each client only stores the presentation of only one user $\mathbf{p}_u$, the update on the item embedding matrix at each local step is the sum of a rank-1 matrix and a regularization component. 
Given that $\lambda$ is typically small, the low-rank component contributes most to the update $\Delta_Q^{(t)}$.
And if the direction of $\mathbf{p}_\usr$ does not change much during the local optimization phase, the update $\Delta_Q^{(t)}$ can stay low-rank.
From this observation, we first assume that the update of the item embedding matrix in training FedRec systems $\Delta_Q^{(t)}$ can be well approximated by a low-rank matrix.
We empirically verify this assumption by monitoring the effective rank of $\Delta_Q^{(t)}$ at each training round for different datasets.
The result is plotted in figure \ref{fig:pinterest_pca} where we plot the mean and standard deviation averaged over a set of participants in each round.

\begin{figure}[h]
    \centering
    \includegraphics[width=0.49\columnwidth]{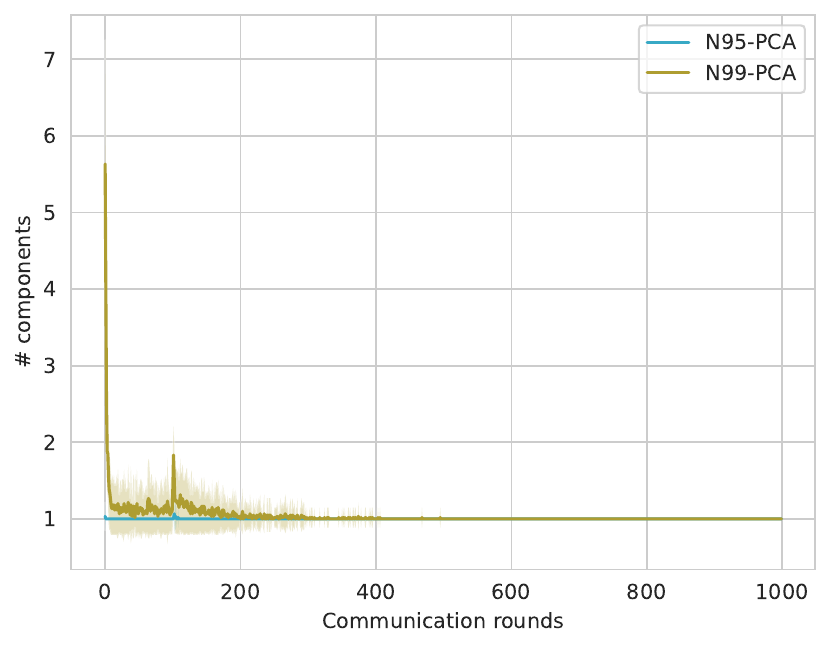}
    \includegraphics[width=0.49\columnwidth]{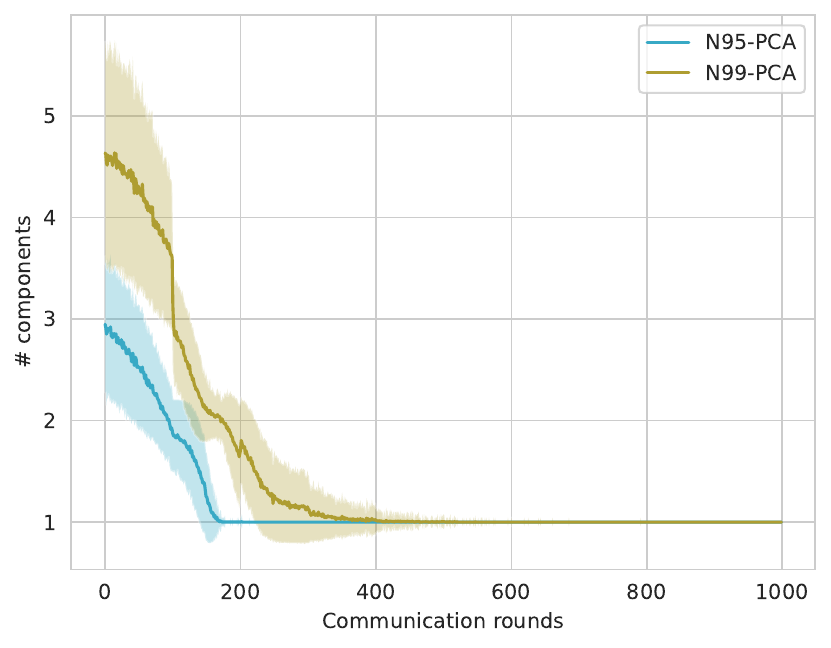}
    \caption{PCA components progression. The figures show the number of components that account for 99\% (N99-PCA in green) and 95\% (N95-PCA in blue) explained variance of all transfer item embedding matrix across communication rounds on the MovieLens-1M (left) and Pinterest (right) datasets. }
    \label{fig:pinterest_pca}
\end{figure}

This analysis suggests that reducing the communication by restricting the update to be low-rank might not sacrifice performance significantly. 
In the next section, we propose an efficient communication framework based on this motivation.
Since most of the transferred parameters in recommendation models are from the item embedding layers, we will focus on applying the proposed method for embedding layers in this work. 

\subsection{Low-rank Structure}
We propose explicitly enforcing a low-rank structure on the local update of the item embedding matrix $Q$. In particular, the local update $(\Delta_Q)_\usr^{(t)}$ is parameterized by a matrix product
\begin{equation*}
    (\Delta_Q)_\usr^{(t)} = B_u^{(t)} A_u^{(t)}
\end{equation*}
where $ B_u^{(t)} \in \mathbb{R}^{\embsz \times \lorarank}$ and $A_u^{(t)} \in \mathbb{R}^{\lorarank \times \numItems}$.
Given this parameterization, the embedding $\mathbf{q}_\itm$ of an item with index $\itm$ is given by
\begin{equation*}
    \mathbf{q}_i^{(t)} = \left(Q^{(t)} + B_u^{(t)} A_u^{(t)}\right) \mathbf{e}_i
\end{equation*}
where $\mathbf{e}_\itm \in \mathbb{R}^{\numItems}$ is a one-hot vector whose value at $\itm$-th is 1. 
This approach effectively saves a factor of $\frac{N\times \embsz }{N \times \lorarank + \embsz \times \lorarank}$ in communication since clients only need to send the much smaller matrices $A_u$ and $B_u$ to the central server. 

\subsection{Correlated Low-rank Structure Update}
Even though enforcing a low-rank structure on the update can greatly reduce the uplink communication size, doing aggregation and performing privacy-preserving is not trivial and faces the following three challenges: (1) the server needs to multiply out all the pairs $A_\usr^{(t)}$ and $B_\usr^{(t)}$ before performing the aggregation step; (2) the sum of low-rank solutions would typically leads to a larger rank update so there is no reducing footprint in the downlink communication; (3) secure aggregation method such as HE cannot directly apply to $A_\usr^{(t)}$ and $B_\usr^{(t)}$ since it will require to perform the multiplication between two encrypted matrices, which is much more costly than simple additive operation.

To reduce the downlink communication cost, we observe that if either $A_\usr^{(t)}$ or $B_\usr^{(t)}$ is identical between users and is fixed during the local training process, then the result of the aggregation step can be represented by a low-rank matrix with the following formulation:
\begin{equation*}
    \Delta_Q^{(t)} = B^{(t)}\left(\sum_{u \in S}A_u^{(t)}\right).
\end{equation*}
Notice that this aggregation is also compatible with HE since it only requires additive operations on a set of $A_\usr^{(t)}$ and clients can decrypt this result and then compute the global update $\Delta_Q^{(t)}$ at their local device.

Based on the above observation, we propose the Correlated Low-rank Structure Update (\arcronym) framework. In this framework, the server randomly initializes a matrix $B^{(t)}$ at the beginning of each training round and shares it among all participants. Participants then set $B_u^{(t)} = B^{(t)}$ and freeze this matrix during the local training phase and only optimize for $A_u^{(t)}$. The framework is presented in Algorithm \ref{algo:cols} and illustrated in Figure \ref{fig:main-method}. 
Note that the communication cost can be further reduced by sending only the random seed of the matrix $B^{(t)}$. \revise{A concurrent work \cite{anonymous2023improving} proposes FFA-LoRA which also fixes the randomly initialized non-zero matrices and only finetunes the zero-initialized matrices. They study FFA-LoRA in the context of federated fine-tuning LLMs and using differential privacy \cite{dwork2006dp} to provide privacy guarantees.}

\begin{algorithm}[ht]
    \DontPrintSemicolon
    \SetKwInput{Input}{Input}
    \SetAlgoLined
    \LinesNumbered
    \Input{Initial model $Q^{(0)}$; update rank $\lorarank$, a distribution $\mathcal{D}_B$ for initializing $B$; $\clientopt$, $\serveropt$ with learning rates $\lr, \slr$;}
    
     \For{$t \in \{0, 1, 2, \dots, T\}$ }{
      Sample a subset $\activeClients^{(t)}$ of clients\;
      Sample $B^{(t)} \sim \mathcal{D}_B$\;
      \For{{\bf client} $\usr \in \activeClients^{(t)}$ {\bf in parallel}}{
        \If{$t > 0$}{
            Download $A^{(t)}$\;
            Merge $Q_\usr^{(t)} = Q^{(t - 1)} + B^{(t - 1)} A^{(t)}$\;
        }
        Initialize $Q_\usr^{(t,0)} = Q^{(t)}$\;
        Download $B^{(t)}$ and Initialize $A_\usr^{(t,0)} = \boldsymbol{0}$\;
        Set trainable parameters $\theta_\usr^{(t,0)} = \{A_\usr^{(t,0)}, \mathbf{p}_\usr^{(t,0)}\}$\;
        \For{$k = 0, \dots, \localStep_\usr-1$}{
        
            % Compute local stochastic gradient $\sgrad \mathcal{L}_\usr(\theta_\usr^{(t,k)})$\;
            Perform local update $\theta_\usr^{(t,k+1)} = \clientopt\left(\theta_\usr^{(t,k)}, \sgrad \mathcal{L}_\usr\left(\theta_\usr^{(t,k)}\right), \lr\right)$\;
        }
        $\mathbf{p}_\usr^{(t+1)} =  \mathbf{p}_\usr^{(t,\tau_\usr)}$\;
        Upload $\{A_\usr^{(t, \tau_\usr)}\}$ to the central server\;
        
      }
      Aggregate local changes
      \begin{equation*}
          A^{(t + 1)} = \sum_{\usr \in \activeClients^{(t)}} \dfrac{N_\usr}{N} A_\usr^{(t, \tau_\usr)};
      \end{equation*}
     }
     \caption{Correlated Low-rank Structure Update Matrix Factorization}
     \label{algo:cols}
\end{algorithm}

\paragraph{Differences w.r.t. SVD compression.}  We compare our method with SVD since it also uses a low-rank structure. The difference is that in \arcronym, participants directly optimize these models on the low-rank parameterization, while SVD only compresses the result from the local training step.

\subsection{Subsampling \fullname~(SCoLR)}
\label{sec:scolr}
In this section, we consider scenarios where edge devices establish communication with a central server using network connections that vary in quality. 
We propose a variant of \arcronym~termed Subsampling \fullname~(SCoLR) which allows each device to choose a unique local rank, denoted as $r_u$, aligning with their specific computational capacities and individual preferences throughout the training process. 
% SCoLR strategically harnesses the more abundant downlink bandwidth.

Let us denote $r_g$ as the rank of global update, which is sent from the server to participants through downlink connections, and $r_\usr$ as the rank of local update, which is sent from clients to the central server for aggregation through uplink connections.
In implementation, we set $r_g$ to be larger than $r_\usr$, reflecting that downlink bandwidth is often higher than uplink.
Given these rank parameters, at the start of each training round, the central server first initializes a matrix $B$ with the shape of $\mathbb{R}^{\embsz \times r_g}$. Then, participants in that round will download this matrix to their local devices and select a subset of columns of $B$ to perform the local optimization step. 
In particular, we demonstrate this process through the following formulation:
\begin{equation}
    (\Delta_Q)_\usr^{(t)} = BS_uA_u,
    \label{eq:lora_heter}
\end{equation}
where $B$ is a matrix with the shape of $\mathbb{R}^{\embsz \times r_g}$ and $A_u$ is a matrix with the shape of $\mathbb{R}^{r_\usr \times \numItems}$.
Specifically, $S_u$ is a binary matrix with $r_\usr$ rows and $r_g$ columns, where each row has exactly one non-zero element. The non-zero element in the $i$-th row is at the $j$-th column, where $j$ is the $i$-th element of a randomly shuffled array of integers from 1 to $r_g$. The detail is presented in Algorithm \ref{algo:algo-subcols}. 
Importantly, sharing the matrix $S_\usr$ does not divulge sensitive user information. Multiplying this matrix with $S_\usr$ is essentially a row reordering operation on the matrix $A_\usr$. As a result, we can effectively perform additive HE between pairs of rows from $A_{\usr_1}$ and $A_{\usr_2}$. This approach ensures privacy while accommodating varying network connections' quality among clients. 

\section{Experiments}

\subsection{Experimental Setup}
\begin{table}[!ht]
\caption{Statistics of the datasets used in evaluation.}
\small
\begin{center}
  \setlength\tabcolsep{4pt}
  \begin{tabular}{lrrrccc}
    \toprule
     \textbf{Datasets}  & \textbf{\# Users}  & \textbf{\# Items} & \textbf{\# Ratings}  & \textbf{Data Density}  \\
    \midrule
    MovieLens-1M~\cite{movielens_2015}  & $6,040$ & $3,706$ & $1,000,209$ & $4.47\%$ \\
    Pinterest~\cite{Geng2015pinterest}  & $55,187$ & $9,916$ & $1,500,809$ & $0.27\%$ \\
    \bottomrule
  \end{tabular}
  \end{center}
  \label{tab:datasets}
\end{table}

\paragraph{Datasets}
We experiment with two publicly available datasets, which are MovieLens-1M \cite{movielens_2015} and Pinterest \cite{Geng2015pinterest}. 
The statistics of these datasets are summarized in Table \ref{tab:datasets}.
We follow common practice in recommendation systems for preprocessing by retaining users with at least 20 interactions and converting numerical ratings into implicit feedback \cite{ammad2019federated, He2017ncf}. 

\paragraph{Evaluation Protocols}
We employ the standard leave-one-out evaluation to set up our test set \cite{He2017ncf}. For each user, we use all their interactions for training while holding out their last interaction for testing. During the testing phase, we randomly sampled 99 non-interacted items for each user and ranked the test item amongst these sampled items.

To evaluate the performance and verify the effectiveness of our model, we utilize two evaluation metrics, i.e., Hit Ratio (HR) and Normalized Discounted Cumulative Gain (NDCG), which are widely adopted for item ranking tasks. 
The above two metrics are usually truncated at a particular rank level (e.g. the first $k$ ranked items) to emphasize the importance of the first retrieved items.
Intuitively, the HR metric measures whether the test item is present on the top-$k$ ranked list, and the NDCG metric measures the ranking quality, which comprehensively considers both the positions of ratings and the ranking precision.

\paragraph{Models and Optimization.} For the base models, we adopt Matrix Factorization with the FedAvg learning algorithm, also used in \citet{chai2020secure}. In our experiments, the dimension of user and item embedding $\embsz$ is set to 64 for the MovieLens-1M dataset and 16 for the Pinterest dataset. This is based on our observation that increasing the embedding size on the Pinterest dataset leads to overfitting and decreased performance on the test set. This observation is also consistent with \citet{He2017ncf}. We use the simple SGD optimizer for local training at edge devices. 

\paragraph{Federated settings.} 
In each round, we sample $M$ clients uniformly randomly, without replacement in a given round and across rounds. Instead of performing $\tau_i$ steps of ClientOpt, we perform $E$ epochs of training over each client’s dataset. This is done because, in practical settings, clients have heterogeneous datasets of varying sizes. Thus, specifying a fixed number of steps can cause some clients to repeatedly train on the same examples, while certain clients only see a small fraction of their data.

\paragraph{Baselines.} We have conducted a comparison between our framework and the basic FedMF models, along with two compression methods: SVD and Top-K compression. The first method, which is SVD-based, returns a compressed update with a low-rank structure. The second method is based on sparsification, representing updates as sparse matrices to reduce the transfer size.

\paragraph{Hyper-parameter settings.} 
To determine hyper-parameters, we create a validation set from the training set by extracting the second last interaction of each user and tuning hyper-parameters on it. We tested the batch size of [32, 64, 128, 256], the learning rate of [0.5, 0.1, 0.05, 0.01], and weight decay in $[5e-4, 1e-4]$. For each dataset, we set the number of clients participating in each round to be equal to 1\% of the number of all users in that dataset. We also vary the local epochs in [1, 2, 4].
The number of aggregation epochs is set at 1000 for MovieLens-1M and 2000 for Pinterest as the training process is converged at these epochs.

\paragraph{Machine} The experiments were conducted on a machine equipped with an Intel(R) Xeon(R) W-1250 CPU @ 3.30GHz and a Quadro RTX 4000 GPU.

\begin{figure}
    \centering
    \includegraphics[width=0.9\columnwidth]{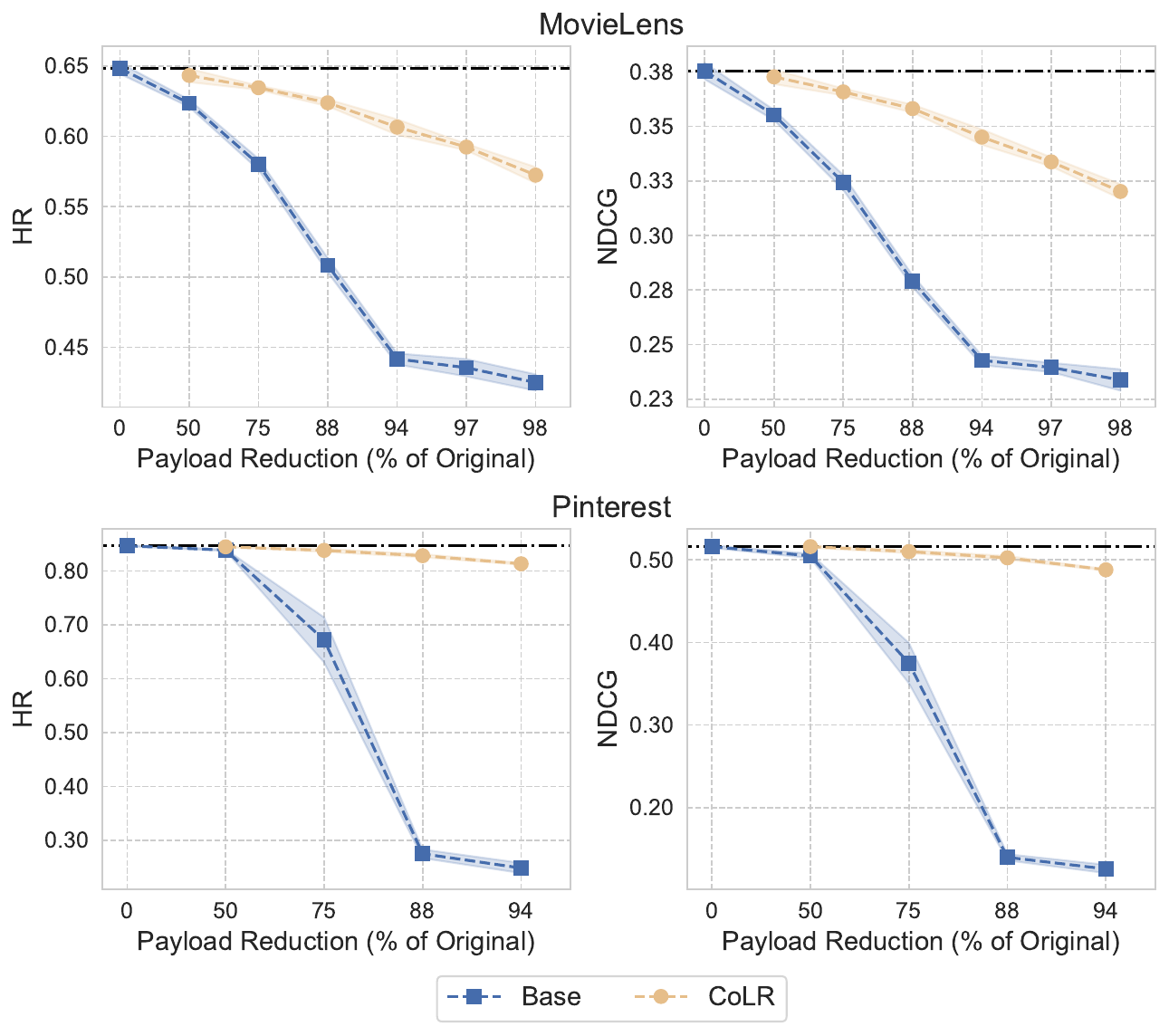} 
    \caption{Performance on the MovieLens-1M dataset (Top) and the Pinterest dataset (Bottom). We plot the utilities (HR and NDCG) versus the payload reduction and compare \arcronym\;with the base model with the same transfer size. Each point represents the average recommendation performance on the test set across five random seeds. The shaded areas denote the standard deviation over the mean. The dashed black line presents the largest base model's performance.}
    \label{fig:transfer-size-comparision}
\end{figure}

\subsection{Experimental Results}

\paragraph{\textbf{(1) \arcronym~can achieve comparable performances with the base models.}}
Given our primary focus is on recommendation performance within communication-limited environments, we commence our investigation by comparing the recommendation performance between \arcronym~ and the base model FedMF given the same communication budget. 
On the ML-1M dataset, we adjust the dimensions of user and item embeddings across the set [1, 2, 4, 8, 16, 32, 64] for FedMF while fixing the embedding size of \arcronym~ to 64, with different rank settings within [1, 2, 4, 8, 16, 32]. Similarly, for Pinterest, the embedding range for FedMF is [1, 2, 4, 8, 16], while \arcronym~ has an embedding size of 16 and ranks in the range of [1, 2, 4, 8]. Our settings lead to approximately equivalent transfer sizes for both methods in each dataset.

In Figure \ref{fig:transfer-size-comparision}, we present the HR and NDCG metrics across different transfer sizes.
With equal communication budgets, \arcronym~consistently outperforms the base models on both datasets. 
To illustrate, on the Pinterest dataset, even with an update size equates to 6.25\% of the largest model, \arcronym~achieves a notable performance (81.03\% HR and 48.50\% NDCG) compared to the base model (84.74\% HR and 51.79\% NDCG) while attaining a much larger reduction in terms of communication cost (16x). In contrast, the FedMF models with corresponding embedding sizes achieve much lower accuracies.
On the MovieLens-1M dataset, we also observe a similar pattern where
\arcronym~consistently demonstrates higher recommendation performance when compared to their counterparts.

The result from this experiment highlights that \arcronym~can achieve competitive performance when compared to the fully-trained model, FedMF while greatly reducing the cost of communication.

\paragraph{\textbf{(2) Comparison between \arcronym~and other compression-based methods.}}

\begin{figure}
    \centering
    \includegraphics[width=0.9\columnwidth]{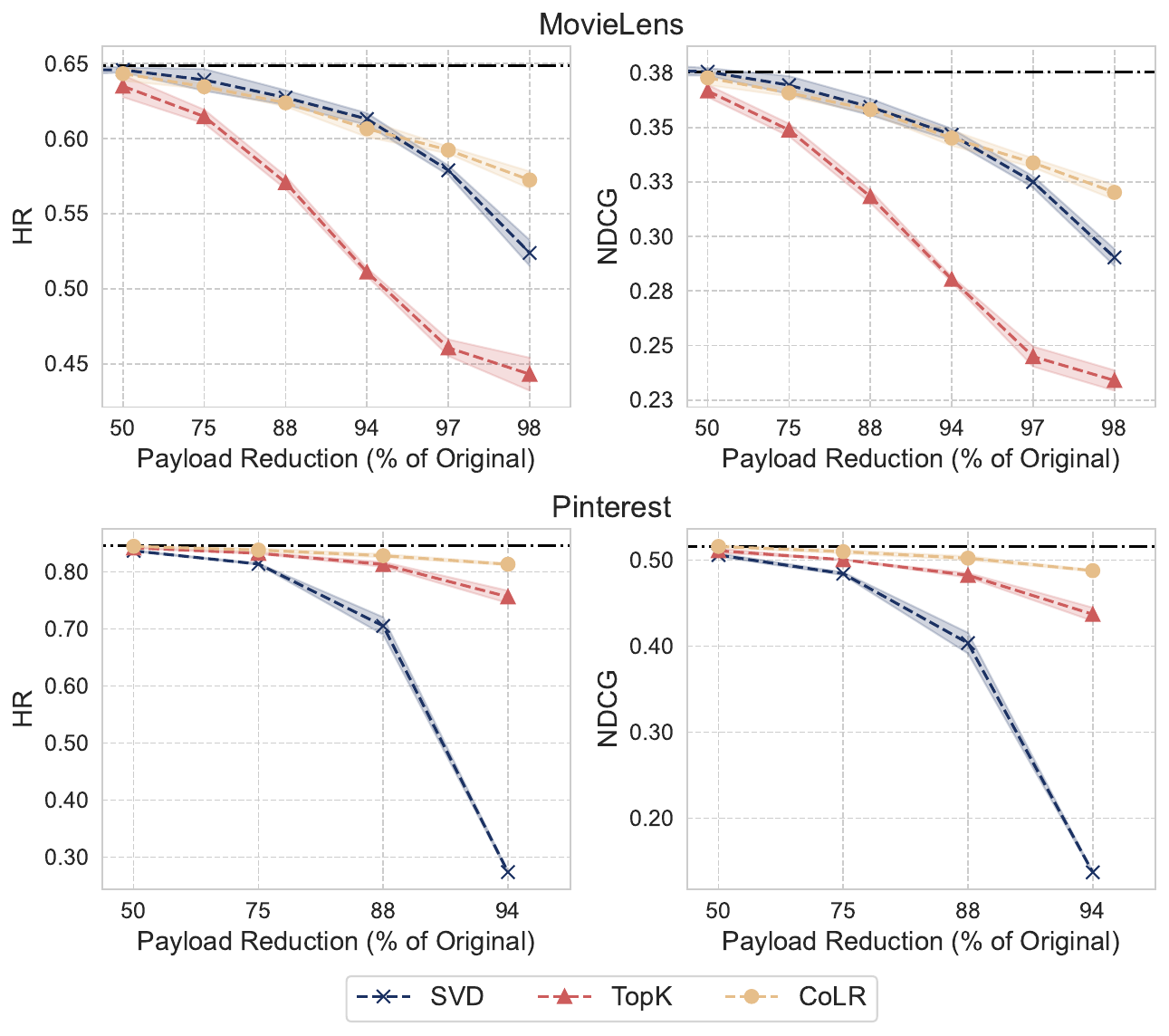} 
    \caption{\textbf{HR and NDCG on MovieLens-1M dataset (Top) and Pinterest Dataset (Bottom).} We plot the utilities versus the payload reduction and compare \arcronym~with other methods with the same payload reduction.  
    The dashed black line presents the base model's performance.
    }
    \label{fig:transfer-size-comparision-wothers}
\end{figure}

We conducted the above experiment employing two compression methods, SVD and top-K compression, with compression ratios matched to those of \arcronym. To ensure a fair evaluation, we applied the same compression ratio to both upload and download messages. The outcomes, depicted in Figure \ref{fig:transfer-size-comparision-wothers}, reveal that \arcronym~ consistently achieves favorable performance while outperforming other methods in scenarios with limited communication budgets. Notably, the performance of SVD and top-K compression varies across datasets. While SVD demonstrates favorable results with the MovieLens dataset, its performance substantially deteriorates with the Pinterest dataset. 

\begin{table}[h!]
\caption{Communication and training times for MovieLens-1M dataset, measured in minutes.}
\small
\begin{center}
\setlength\tabcolsep{4.5pt}
    \begin{tabular}{p{0.5in}p{0.8in}p{0.7in}p{0.7in}}
    % \begin{tabular}{p{0.5in}m{0.8in}m{0.7in}m{0.7in}}
        \toprule
        \textbf{Method}  & \textbf{Communication time} (mins) & \textbf{Computation time}  (mins) & \textbf{Total Training Time} (mins)\\
        \midrule
        MF-64 & 80.43 & 169.07 & 249.50 \\
        \midrule
        % FedAvg-16 & 50\% & 4.48 & 14.46 & 18.95 \\
        \arcronym@1 & 1.26 & 169.18 & 170.43 \\
        \arcronym@2 & 2.51 & 169.21 & 171.72 \\ 
        \arcronym@4 & 5.03 & 169.27 & 174.30 \\
        \arcronym@8 & 10.05 & 169.29 & 179.34 \\
        \arcronym@16 & 20.11 & 169.30 & 189.41 \\
        \arcronym@32 & 40.21 & 169.38 & 209.60 \\
        \midrule
        % FedAvg-16 & 50\% & 4.48 & 14.46 & 18.95 \\
        SVD@1 & 1.26 & 169.49 & 170.75 \\
        SVD@2 & 2.51 & 169.50 & 172.02 \\ 
        SVD@4 & 5.03 & 169.53 & 174.55 \\
        SVD@8 & 10.05 & 169.59 & 179.65 \\
        SVD@16 & 20.11 & 169.64 & 189.74 \\
        SVD@32 & 40.21 & 169.60 & 209.82 \\
        \midrule
        Top-K@1 & 2.51 & 169.76 & 172.28 \\
        Top-K@2 & 5.03 & 169.79 & 174.81 \\
        Top-K@4 & 10.05 & 169.82 & 179.87 \\
        Top-K@8 & 20.11 & 169.92 & 190.03 \\ 
        Top-K@16 & 40.21 & 170.14 & 210.35 \\ 
        \bottomrule
    \end{tabular}
\end{center}
\label{tab:computation_time}
\end{table}

In the previous results, the evaluation of techniques focuses on the overall number of transmitted bits. Although this serves as a broad indicator, it fails to consider the time consumed by encoding/decoding processes and the fixed network latency within the system. When these time delays significantly exceed the per-bit communication time, employing compression techniques may offer limited or minimal benefits.
In the following, we do an analysis to understand the effects of
using \arcronym~and compression methods in training FedRec models.

We follow the model from \cite{wang2021fieldguide} to estimate the communication efficiency of deploying methods to real-world systems.
The execution time per round when deploying an optimization algorithm $\mathcal{A}$ in a cross-device FL system is estimated as follows, 
\begin{align*}
& T_{\text {round }}(\mathcal{A})=T_{\text {comm }}(\mathcal{A})+T_{\text {comp }}(\mathcal{A}), \\
& T_{\text {comm }}(\mathcal{A})=\frac{S_{\text {down }}(\mathcal{A})}{B_{\text {down }}}+\frac{S_{\text {up }}(\mathcal{A})}{B_{\text {up }}} \\
& T_{\text {comp }}(\mathcal{A})=\max _{j \in \mathcal{D}_{\text {round }}} T_{\text {client }}^j+T_{\text {server }}(\mathcal{A}), \\
& T_{\text {client }}^j(\mathcal{A})=R_{\text {comp }} T_{\text {sim }}^j(\mathcal{A})+C_{\text {comp }} \\
&
\end{align*}
where client download size $S_{\text down}(\mathcal{A})$, upload size $S_{\text up}(\mathcal{A})$, server computation time $T_{\text {server }}$, and client computation time $T_{\text {client }}^j$ depend on model and algorithm $\mathcal{A}$.  
Simulation time $T_{\text {server }}$ and $T_{\text {client }}^j$ can be estimated from FL simulation in our machine. We get the estimation of parameters $(B_{\text {down }}, B_{\text {up }}), R_{\text {comp }}, C_{\text {comp }}$ from \citet{wang2021fieldguide}.

\begin{align*}
&B_{\text {down }} \sim 0.75 \mathrm{MB} / \mathrm{secs}, B_{\mathrm{up}} \sim 0.25 \mathrm{MB} / \mathrm{secs}, \\
&R_{\mathrm{comp}} \sim 7, \text { and } C_{\text {comp }} \sim 10 \text { secs. }
\end{align*}

Table \ref{tab:computation_time} presents our estimation in terms of communication times and computation time. Notice that \arcronym~adds smaller overheads to the computation time while still greatly reducing the communication cost.

\begin{table*}[h!]
  \caption{Overheads, and Communication ratios for MovieLens-1M dataset; Comm Ratio is calculated by file sizes of Ciphertext over file sizes of Plaintext.}
  \label{tab:he_results}
\small
\begin{center}
  \begin{tabular}{lrrrrrr}
    \toprule \textbf{ Method} & \textbf{Client overheads}  & \textbf{Server overheads} & \textbf{Ciphertext size} & \textbf{Plaintext size} & \textbf{Comm Ratio}\\
    \midrule
      FedMF  & 0.93 s & 2.39 s & 24,587 KB & 927 KB & 26.52\\
      \midrule
      FedMF w/ Top-K@1/64& 88.20 s                 & 88.06 s                & 3,028 KB  & 29 KB  & 103.09  \\
      FedMF w/ Top-K@2/64& 182.02 s              & 185.59 s               & 6,056 KB  & 58 KB  & 103.83 \\
      FedMF w/ Top-K@4/64& 353.25 s              & 364.67 s               & 12,112 KB & 116 KB  & 104.20 \\
      FedMF w/ Top-K@8/64& 723.45 s               & 750.98 s               & 24,225 KB & 232 KB & 104.40 \\
      FedMF w/ Top-K@16/64& 1449.90 s              & 1483.91 s            & 48,448 KB & 464 KB & 104.49 \\
      % FedMF w/ TopK@32/64& 2957.79              & 2,950.32             & 96,927 KB & 927 KB & 187.27 \\
      \cmidrule{1-6}
      FedMF w/ \arcronym@1                & 0.07 s                & 0.24 s                & 3,073 KB  & 15 KB  & 206.31 \\
      FedMF w/ \arcronym@2                & 0.07 s                & 0.25 s                & 3,073 KB  & 29 KB  & 104.63 \\
      FedMF w/ \arcronym@4                & 0.07 s                & 0.25 s                 & 3,073 KB  & 58 KB  & 52.69  \\
      FedMF w/ \arcronym@8                & 0.08 s                & 0.25 s                & 3,073 KB  & 116 KB & 26.44  \\
      FedMF w/ \arcronym@16               & 0.15 s                & 0.51 s                & 6,147 KB  & 232 KB & 26.49  \\
      FedMF w/ \arcronym@32               & 0.30 s                & 1.03 s                & 12,293 KB & 464 KB & 26.51  \\
    \bottomrule
  \end{tabular}
  \end{center}
  \label{tab:results-lora-he-en}
\end{table*}

\paragraph{\textbf{(3) \arcronym~is compatible with HE}}

In this section, we argue that tackling privacy and communication efficiency as separate concerns can result in suboptimal solutions and point out the limitation in applying SVD and Top-K compression on HE-based FedRec systems.

Since performing SVD decomposition on an encrypted matrix remains an open problem, we conduct tests using two communication efficient methods: \arcronym~and Top-K. These tests are carried out under identical configurations, encompassing local updates and the number of clients involved in training rounds. The setup entails the utilization of the CKKS Cryptosystem \cite{cheon2017ckks} for our \arcronym~method, while the Top-K method employs the Paillier cryptosystem \cite{paillier1999} for encryption, decryption, and aggregation in place of the top-K vector. 
The detailed implementation is described in Appendix \ref{sec:he_implementation}. Table \ref{tab:he_results} displays the client and server overheads in seconds, as well as the size of the ciphertext and plaintext.

Table \ref{tab:he_results} shows that CoLR can reduce client and server overheads by up to 
$3\text{-}10\times$. 
For Top-K compression, when the value of $k$ doubles (i.e., doubling the top-K vector's size), the operation time for both client-side and server-side operations also doubles, as it mandates operations on each value within the vector. Throughout the experiment, \arcronym~ consistently outperforms the Top-K method across compression ratios, exhibiting lower time overheads on both the client and server sides.
In terms of ciphertext sizes, the Top-K compression method with Paillier encryption demands encryption for each value within the top-K vector. Consequently, whenever the size of the top-K vector doubles, the ciphertext size also doubles. 
In contrast, as previously explained, our scheme produced at most $\lceil \frac{n}{8096} \rceil$ blocks of ciphertext, with the ciphertext size not doubling each time $k$ doubles. This phenomenon illustrates why, in several cases, the ciphertext size remains consistent even as the plaintext size increases. With lower payload reductions aimed at achieving greater recommendation performance, our scheme demonstrates smaller ciphertext sizes, offering a reduction in bandwidth consumption.

\subsection{Heterogeneous network bandwidth}

In this section, we evaluate our proposed method SCoLR and explore the scenario where each client can dynamically select $\lorarank_u$ value during each training round $t$. This scenario reflects real-world federated learning, where clients often showcase differences in communication capacities, as exemplified in \cite{fan2021oort, li2020fedchallenges}. It becomes inefficient to impose a uniform communication budget on all clients within this heterogeneous context, as some devices may not be able to harness their network connections fully.

For this experiment, we set the global rank $r_g$ of SCoLR in the list of $\{2, 4, 8, 16, 32, 64\}$ and uniformly sample the local rank $r_\usr$ such that $1 \le r_\usr \le r_g$. It's crucial to emphasize that $r_\usr$ is independently sampled for each user and may differ from one round to the next. This configuration mirrors a practical scenario where the available resources of a specific user may undergo substantial variations at different time points during the training phase. We present the result on the MovieLens-1M dataset in Table \ref{tab:results-scolr-hetero}. 
We compare S\arcronym~with the base models and \arcronym~in Figure \ref{fig:scolr-results}.
This result demonstrates that SCoLR is effective under the device heterogeneity setting since it can match the performance of \arcronym~under the same uplink communication budget. 

\begin{figure}[h!]
    \centering
    \includegraphics[width=0.9\columnwidth]{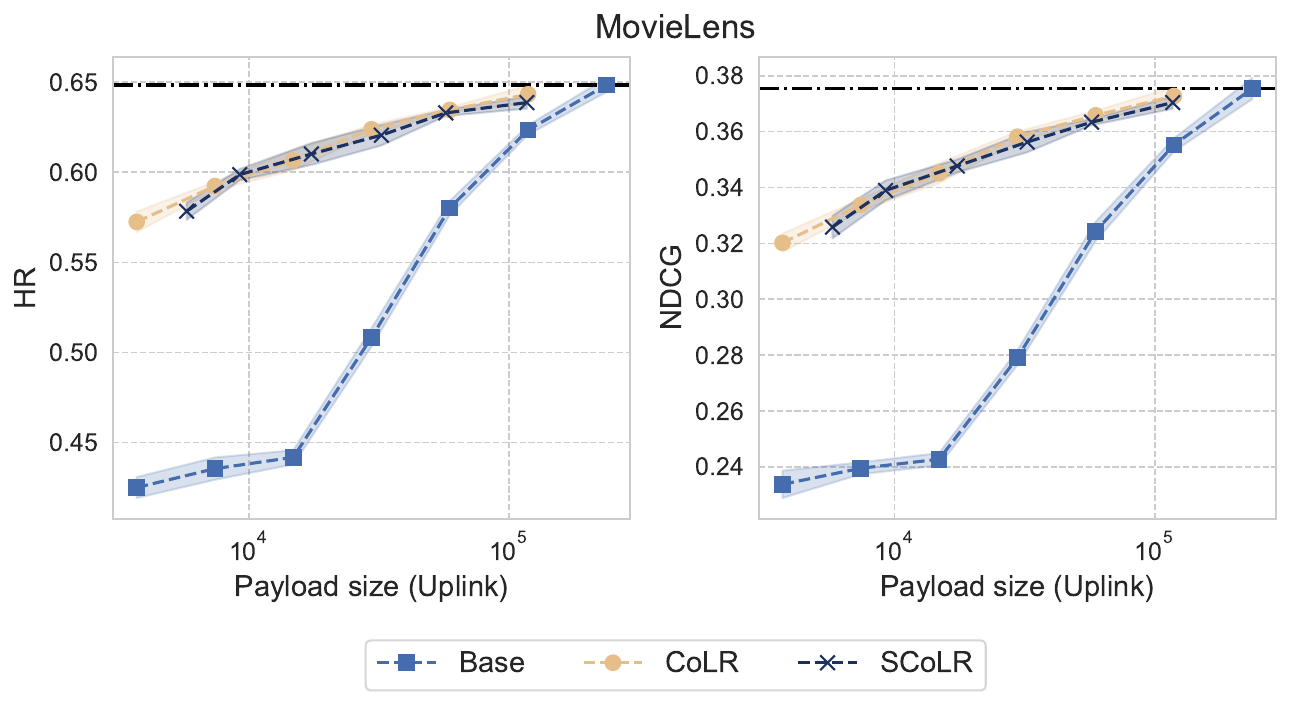} 
    \caption{\textbf{Performance of SCOLR on MovieLens-1M dataset.} We plot the utilities versus the payload size.  
    The dashed black line is the base model's performance.
    }
    \label{fig:scolr-results}
\end{figure}

\section{Conclusion}

In this work, we propose Correlated Low-rank Structure update (CoLR), a framework that enhances communication efficiency and privacy preservation in FedRec by leveraging the inherent low-rank structure in updating transfers, our method reduces communication overheads.
\arcronym~also benefits from the CKKS cryptosystem, which allows the implementation of a secured aggregation strategy within FedRec.
With minimal computational overheads and bandwidth-heterogeneity awareness, it offers a flexible and efficient means to address the challenges of federated learning. For future research, we see several exciting directions. First, our framework still involves a central server, we would like to test how our methods can be effectively adapted to a fully decentralized, peer-2-peer communication setting \cite{zhang2020, lyu2022}. Secondly, investigating methods to handle dynamic network conditions and straggler mitigation in real-world settings will be crucial. Lastly, expanding our approach to accommodate more advanced secure aggregation techniques for reduced server-side computational costs and extending its compatibility with various encryption protocols can further enhance its utility in privacy-sensitive applications.

%% file: appendix.tex
\section{Algorithm Details}
In Section \ref{sec:scolr}, we presented SCoLR to address the bandwidth heterogeneity problem. We provide the detail of this method in Algorithm \ref{algo:algo-subcols} and the detail experimental results in Table \ref{tab:results-scolr-hetero}.

\begin{algorithm}[ht]
    \DontPrintSemicolon
    \SetKwInput{Input}{Input}
    \SetAlgoLined
    \LinesNumbered
\Input{Initial model $Q^{(0)}$; global update rank $\lorarank_g$, local update rank $\{\lorarank_\usr\}$, a distribution $\mathcal{D}_B$ for initializing $B$; $\clientopt$, $\serveropt$ with learning rates $\lr, \slr$;}
    
     \For{$t \in \{0, 1, 2, \dots, T\}$ }{
      Sample a subset $\activeClients^{(t)}$ of clients and $B^{(t)} \sim \mathcal{D}_B$\;
      \For{{\bf client} $\usr \in \activeClients^{(t)}$ {\bf in parallel}}{
        \If{$t > 0$}{
            Download $A^{(t)}$\ and merge $Q_\usr^{(t,0)} = Q^{(t - 1)} + B^{(t - 1)} A^{(t)}$\;
        }
        Initialize $Q_\usr^{(t,0)} = Q^{(t)}$\;
        Download $B^{(t)}$, Initialize $A_\usr^{(t,0)} = \boldsymbol{0}$, and  sample $S_\usr^{(t)}$\;
        Set trainable parameters $\theta_\usr^{(t,0)} = \{A_\usr^{(t,0)}, \mathbf{p}_\usr^{(t,0)}\}$\;
        \For{$k = 0, \dots, \localStep_\usr-1$}{
        
            % Compute local stochastic gradient $\sgrad \mathcal{L}_\usr(\theta_\usr^{(t,k)})$\;
            Perform local update $\theta_\usr^{(t,k+1)} = \clientopt\left(\theta_\usr^{(t,k)}, \sgrad_{\theta_\usr} \mathcal{L}_\usr(\theta_\usr^{(t,k)}), \lr\right)$\;
        }
        $\mathbf{p}_\usr^{(t+1)} =  \mathbf{p}_\usr^{(t,\tau_\usr)}$\;
        Upload $\{S_\usr^{(t)}, A_\usr^{(t, \tau_\usr)}\}$ to the central server\;
        
      }
      Aggregate local changes
      \begin{equation*}
          A^{(t + 1)} = \sum_{\usr \in \activeClients^{(t)}} \dfrac{N_\usr}{N} S_\usr^{(t)} A_\usr^{(t, \tau_\usr)};
      \end{equation*}
     }
     \caption{Subsampling \fullname~(SCoLR)}
     \label{algo:algo-subcols}
\end{algorithm}

\begin{table}[!hbt]
\caption{HR, NDCG of SCoLR algorithm on the MovieLens-1M dataset under computation/device heterogeneity settings.}
\small
\begin{center}
   % \setlength\tabcolsep{3.5pt}
    % \scalebox{0.9}{
   \begin{tabular}{p{0.8in}p{0.6in}lll}
       \toprule
       \textbf{Global rank}  & \textbf{Local rank}  & \textbf{HR} & \textbf{NDCG}\\
       \midrule
       64 & $1-64$ & 63.86 & 37.04\\
       32 & $1-32$ & 63.29 & 36.34\\
       16 & $1-16$ & 62.06 & 35.63\\
       8 & $1-8$ & 61.02 & 34.77\\
       4 & $1-4$ & 59.87 & 33.89\\
       2 & $1-2$ & 57.84 & 32.58\\
       \bottomrule
   \end{tabular}
    % }    
\end{center}
\label{tab:results-scolr-hetero}
\end{table}
\section{Experimental Details}
We plot the convergence speed of four methods on the Pinterest dataset in Figure \ref{fig_app:pinterest_train_hist}.

\begin{figure}[h!]
    \centering
    \includegraphics[width=1.0\columnwidth]{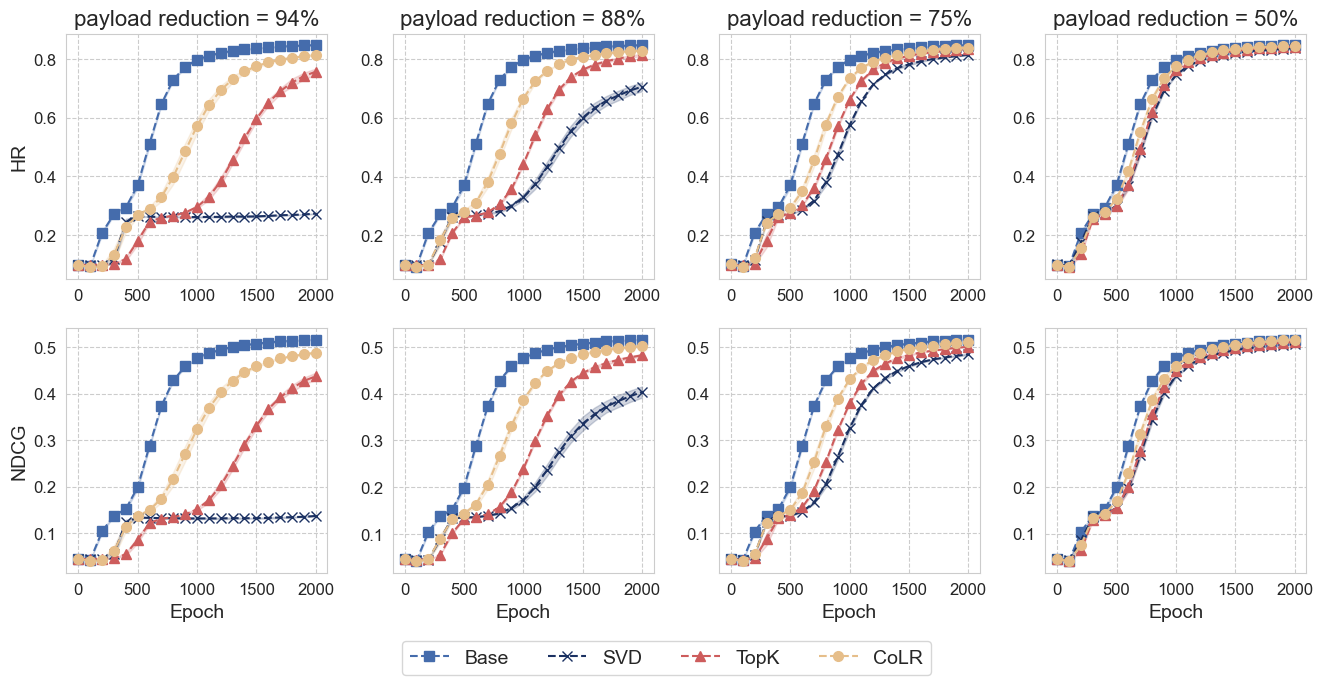} 
    \caption{Average HR and NDCG on Pinterest dataset varies as communication rounds.
    }
    \label{fig_app:pinterest_train_hist}
\end{figure}
\section{Homomorphic Encryption with compressors and \arcronym}
\label{sec:he_implementation}

\textit{Limitation of applying HE with SVD Compression:}
The SVD method requires matrix multiplication on encrypted matrices U, S, and V derived from local clients' updates. There are several research endeavors aimed at providing efficient algorithms for applying homomorphic encryption in this context, specifically using the CKKS cryptosystem \cite{jiang2018hematrixmul}. However, a limitation of this method is that the dimension of the matrix must be in the form of $2^n$, often necessitating additional padding on the original matrices to achieve this form, particularly in the case of larger dimension matrices. Additionally, to reduce the size of global updates sent from the server to clients, additional SVD decompositions are required. Performing SVD decomposition on an encrypted matrix by known schemes remains an open problem, resulting in high downlink bandwidth consumption.

\textit{Limitation of applying HE with TopK Compression:}
The TopK compression method necessitates in-place homomorphic operations, a characteristic not compatible with the CKKS scheme, designed to perform homomorphic encryption on tensors. As an alternative, we have employed the Paillier cryptosystem \cite{paillier1999}, a partially homomorphic encryption scheme capable of encrypting individual numbers. While Paillier allows for the implementation of the FedAvg aggregation strategy, it requires the secured aggregation process on the server must be executed on each element in the TopK vector. Consequently, increasing the value of $K$ results in higher operational costs for encryption, decryption, and secured aggregation.

\textit{Implementation of HE with \arcronym:} Our \arcronym~method leverages the inherent efficiency of the CKKS cryptosystem, which can execute operations on multiple values as a vector. 
For a flattened vector of size $n$, both clients and the server need to perform operations on at most $\lceil \frac{n}{8096} \rceil$ blocks.

\section{An analysis on the initialization of the matrix B}
\label{sec:error_analysis}
If each client performs only one GD step locally then $B$ can be seen as the projection matrix and $B \mathbf{a}_i$ is the projection of the update of item $i$ on the subspace spanned by columns of $B$.
We denote the error of the update on each item embedding $\itm$ by $\epsilon_i$ which has the following formulation:
\begin{equation}
   \epsilon_\itm = \mathbb{E}_{B} \left[ \left\| \bar{\Delta}_Q - \frac{1}{|S|} \left(\sum_{u \in S}B_u \mathbf{a}_\usr\right) \right\|_2^2 \right].
\end{equation}
We analyze the effect of different initialization of $B$ on this error.
First, we state the proposition \ref{lem:bound_error_general} which gives an upper bound on the error $\epsilon_i$.

\begin{proposition}[Upper bound the error]
If $B_\usr$ is independently generated between users and are chosen from a distribution $\mathcal{B}$ that satisfies:
\begin{enumerate}
    \item Bounded operator norm: $\mathbb{E}\left[\|B\|^2\right] \le L_\mathcal{B}$ 
    \item Bounded bias: $\|\mathbb{E} B_{\usr} B_{\usr}^\top \bar{\mathbf{p}}_{\usr} - \bar{\mathbf{p}}_{\usr} \|_2 \le \sqrt{\delta_\mathcal{B}}$
\end{enumerate}
Then,
    \begin{align}
   \epsilon_\itm &= \mathbb{E}_{B} \left[ \left\| \bar{\Delta}_Q - \frac{1}{|S|} \left(\sum_{u \in S}B_u \mathbf{a}_\usr\right) \right\|_2^2 \right] \\ 
   &\le \frac{1}{|S|} C_\mathbf{p}^2 \delta_\mathcal{B} + \frac{1}{|S|} \max_{\usr \in S} \alpha_u \|\mathbf{p}_u\|_2^2  \left( L_B^2 + 1 \right).
\end{align}
\label{lem:bound_error_general}
\end{proposition}

\begin{proof}
    Assume $B_\usr$ is independently generated between users, we have
\begin{align*}
    \epsilon_\itm 
    &= \frac{1}{|S|^2} \mathbb{E}_{B} \left[ \left\| \sum_{u \in S} (r_{\usr\itm}- \hat{r}_{\usr\itm}) \left( B_\usr B_\usr^\top \mathbf{p}_\usr - \mathbf{p}_\usr \right) \right\|_2^2 \right]\\
    &= \frac{1}{|S|^2} \mathbb{E}_{B} \left[ \left\| \sum_{u \in S} \alpha_u \left( B_\usr B_\usr^\top \mathbf{p}_\usr - \mathbf{p}_\usr \right) \right\|_2^2 \right]\\
    &= \frac{1}{|S|^2} \sum_{\usr_1 \in S} \sum_{\usr_2 \ne \usr_1} \alpha_{\usr_1} \alpha_{\usr_2} \mathbb{E}_{B} \left< B_{\usr_1} B_{\usr_1}^\top \mathbf{p}_{\usr_1} - \mathbf{p}_{\usr_1},  B_{\usr_2} B_{\usr_2}^\top \mathbf{p}_{\usr_2} - \mathbf{p}_{\usr_2} \right>\\
    &+ \frac{1}{|S|^2}  \sum_{u \in S}  \alpha^2_{\usr} \mathbb{E}_{B} \left[  \left\| B_\usr B_\usr^\top \mathbf{p}_\usr - \mathbf{p}_\usr \right\|_2^2 \right]
\end{align*}
If $B_u$ are independently chosen from a distribution $\mathcal{B}$ that satisfies:
\begin{enumerate}
    \item Bounded operator norm: $\mathbb{E}\left[\|B\|^2\right] \le L_\mathcal{B}$ 
    \item Bounded bias: $\|\mathbb{E} B_{\usr} B_{\usr}^\top \bar{\mathbf{p}}_{\usr} - \bar{\mathbf{p}}_{\usr} \|_2 \le \sqrt{\delta_\mathcal{B}}$
\end{enumerate}

We have 
\begin{align}
    &\mathbb{E}_{B} \left< B_{\usr_1} B_{\usr_1}^\top \mathbf{p}_{\usr_1} - \mathbf{p}_{\usr_1},  B_{\usr_2} B_{\usr_2}^\top \mathbf{p}_{\usr_2} - \mathbf{p}_{\usr_2} \right> \nonumber \\
    &= \left\| \mathbf{p}_{\usr_1} \right\| \left\| \mathbf{p}_{\usr_2} \right\| \mathbb{E}_{B} \left< B_{\usr_1} B_{\usr_1}^\top \bar{\mathbf{p}}_{\usr_1} - \bar{\mathbf{p}}_{\usr_1},  B_{\usr_2} B_{\usr_2}^\top \bar{\mathbf{p}}_{\usr_2} - \bar{\mathbf{p}}_{\usr_2} \right>   \nonumber \\
    &= \left\| \mathbf{p}_{\usr_1} \right\| \left\| \mathbf{p}_{\usr_2} \right\| \left< \mathbb{E} B_{\usr_1} B_{\usr_1}^\top \bar{\mathbf{p}}_{\usr_1} - \bar{\mathbf{p}}_{\usr_1},  \mathbb{E} B_{\usr_2} B_{\usr_2}^\top \bar{\mathbf{p}}_{\usr_2} - \bar{\mathbf{p}}_{\usr_2} \right>\\ \label{eq:ana_b_inde}
    &\le \left\| \mathbf{p}_{\usr_1} \right\| \left\| \mathbf{p}_{\usr_2} \right\| \|\mathbb{E}  B_{\usr_1} B_{\usr_1}^\top \bar{\mathbf{p}}_{\usr_1} - \bar{\mathbf{p}}_{\usr_1}\|_2 \| \mathbb{E} B_{\usr_2} B_{\usr_2}^\top \bar{\mathbf{p}}_{\usr_2} - \bar{\mathbf{p}}_{\usr_2} \|_2 \nonumber  \\
    &\le C_\mathbf{p}^2 \delta_\mathcal{B}
\end{align}
where (\ref{eq:ana_b_inde}) follows since $B_u$ are independently sampled between users.
The second term is 
\begin{align*}
    & \frac{1}{|S|^2}  \sum_{u \in S}  \alpha^2_{\usr} \mathbb{E}_{B} \left[  \left\| B_\usr B_\usr^\top \mathbf{p}_\usr - \mathbf{p}_\usr \right\|_2^2 \right] \\
    &= \frac{1}{|S|^2}  \sum_{u \in S}  \alpha^2_{\usr} \|\mathbf{p}_u\|_2^2 \mathbb{E}_{B} \left[  \left\| B_\usr B_\usr^\top \bar{\mathbf{p}}_\usr - \bar{\mathbf{p}}_\usr \right\|_2^2 \right] \\
    &= \frac{1}{|S|^2}  \sum_{u \in S} \alpha_u \|\mathbf{p}_u\|_2^2 \mathbb{E}_{B} \left[  \left\| B_\usr B_\usr^\top \bar{\mathbf{p}}_\usr \|_2^2 + \|\bar{\mathbf{p}}_\usr \right\|_2^2 - 2 \bar{\mathbf{p}}_\usr^\top B_\usr B_\usr^\top \bar{\mathbf{p}}_\usr \right]\\
    &= \frac{1}{|S|^2}  \sum_{u \in S} \alpha_u \|\mathbf{p}_u\|_2^2 \mathbb{E}_{B} \left[  \left\| B_\usr B_\usr^\top \bar{\mathbf{p}} \right\|_2^2 + 1 - 2 \left\| B_\usr^\top \bar{\mathbf{p}} \right\|_2^2 \right] \\
    &\le \frac{1}{|S|} \max_{\usr \in S} \alpha_u \|\mathbf{p}_u\|_2^2  \left( \mathbb{E}_{B} \left[  \left\| B_\usr B_\usr^\top \right\|_2^2 \right] + 1 \right)  \\
    &\le \frac{1}{|S|} \max_{\usr \in S} \alpha_u \|\mathbf{p}_u\|_2^2  \left( L_B^2 + 1 \right)
    \end{align*}
\end{proof}

Next, we bound the bias and the operator norm of $B_\usr$ if it is sampled from a Gaussian distribution in the lemma \ref{lem:gaussian}.
\begin{lemma}[Gaussian Initialization]
    Let $r < d$. Consider $B \in \mathbb{R}^{d \times r}$ be sampled from the Gaussian distribution where $B$ has i.i.d. $\mathcal{N}(0, 1/k)$ entries and a fixed unit vector $\mathbf{v} \in \mathbf{R}^{d}$. Then 
    \begin{enumerate}
        \item Bounded operator norm:
            \begin{equation*}
                \mathbb{E} \| B \|^2 \le \frac{d}{r} \left( 1 + O\left(\sqrt{\frac{r}{d}}\right)\right)
            \end{equation*}
        \item  Unbias: for every unit vector $\mathbf{v} \in \mathbb{R}^{d}$
        \begin{equation*}
            \left\| \mathbb{E} B B^\top \mathbf{v} - \mathbf{v} \right\| = 0
        \end{equation*}
    \end{enumerate}
    \label{lem:gaussian}
\end{lemma}
\begin{proof}
    Let $B' = P B $ where $P \in \mathbb{R}^{d \times d}$ is the rotation matrix such that $P\mathbf{v} = \mathbf{e}_1$. Due to the rotational symmetry of the normal distribution, $B'$ is a random matrix with i.i.d. $\mathcal{N}(0,1/r)$ entries. Note that $B = P^\top B'$.
    \begin{align*}
        \mathbb{E}_B \left[ B B^\top \mathbf{v} \right] &= \mathbb{E}_B \left[ P^\top P B B^\top P^\top P \mathbf{v} \right] \\
        &=  P^\top \mathbb{E}_B \left[B' B'^{\top} \mathbf{e}_1 \right]
    \end{align*}
    Let $\mathbf{z} = B' B'^{\top} \mathbf{e}_1$. Notice that $\mathbf{z}_j = \left< B'^\top \mathbf{e}_j, B'^\top \mathbf{e}_1 \right>$. Because $B'$ has i.i.d. $\mathcal{N}(0, 1/r)$ entries,
    $\mathbf{z}_1 = \| B'^\top \mathbf{e}_1 \|_2^2 = \sum_{k = 1}^r (B'_{1k})^2$ is $1 / r$ times a Chi-square random variable with $r$ degrees of freedom. So $\mathbb{E}[\mathbf{z}_1] =  \frac{1}{r} r = 1$ and $\mathbb{E}[\mathbf{z}_j] = 0 \forall j > 1$. Thus, $\mathbb{E}[\mathbf{z}] = \mathbf{e}_1$. Therefore,  $\left\| \mathbb{E}_B \left[ B B^\top \mathbf{v} \right] \right\| = \left\| P^\top \mathbf{e}_1 - \mathbf{v}\right\| = 0.$
\end{proof}

From Proposition \ref{lem:bound_error_general} and Lemma \ref{lem:gaussian}, we can directly get the following theorem which bould the error of restricting the local update in a low-rank subspace which is randomly sampled from a normal distribution. 
\begin{theorem}
    Assume $B_\usr$ is independently generated between users and are chosen from the normal distribution $\mathcal{N}(0, 1 / \lorarank)$.
    Then,
    \begin{align*}
   \epsilon_\itm &= \mathbb{E}_{B} \left[ \left\| \bar{\Delta}_Q - \frac{1}{|S|} \left(\sum_{u \in S}B_u \mathbf{a}_\usr\right) \right\|_2^2 \right] \\ 
   &\le \frac{1}{|S|} \max_{\usr \in S} \alpha_u \|\mathbf{p}_u\|_2^2 ~O\left(\frac{d}{r}\right).
\end{align*}
\end{theorem}

This result demonstrates that the square error can increase for lower values of the local rank $\lorarank$. Building on this insight, we suggest scaling the learning rate of the low-rank components by $\sqrt{\frac{\lorarank}{d}}$ to counter the error.